\documentclass[letterpaper, 10 pt, conference]{ieeeconf}  

\usepackage{blindtext}
\usepackage{graphicx}
\usepackage{epstopdf}
\usepackage{float}
\usepackage{setspace}
\usepackage{amsmath}

\usepackage{mathrsfs}
\usepackage{amssymb}
\usepackage[noadjust]{cite}
\usepackage{multirow}
\usepackage{array}
\usepackage{inputenc}
\let\labelindent\relax
\usepackage[inline]{enumitem}
\usepackage{soul}
\usepackage{booktabs}
\usepackage{varwidth}
\usepackage{bm}

\usepackage{hyperref}
\usepackage{tikz}

\usepackage{amsthm}

\newtheorem{lemma}{{Lemma}}

\newtheorem{remark}{{Remark}}
\theoremstyle{definition}
\newtheorem{definition}{{Definition}}

\newcommand{\bbm}{\begin{bmatrix}}
\newcommand{\ebm}{\end{bmatrix}}

\newcommand{\lone}{{\mathcal{L}_1}}

\newcommand{\bff}[2]{{\bf{#1}}_{#2}}
\newcommand{\hbff}[2]{\widehat{\bf{#1}}_{#2}}

\newcommand{\tbff}[2]{\widetilde{\bf{#1}}_{#2}}

\newcommand{\rev}[1]{{#1}}

\setlist*[enumerate,1]{%
  label=(\roman*),
}

\IEEEoverridecommandlockouts                              

\title{\rev{Data-Efficient Multi-Robot, Multi-Task Transfer Learning for Trajectory Tracking}}

\author{Karime Pereida, Mohamed K. Helwa and Angela P. Schoellig
\thanks{Manuscript received: September, 10, 2017; Revised December, 16, 2017; Accepted December, 29, 2017.}
\thanks{This paper was recommended for publication by Editor Paolo Rocco upon evaluation of the Associate Editor and Reviewers' comments. 
This work was supported in part by the Mexican National Council of Science and Technology (abbreviated CONACYT), by OCE/SOSCIP TalentEdge Project \#27901 and by the Ontario Early Researcher Award.} 
\thanks{Karime Pereida, Mohamed K. Helwa and Angela P. Schoellig with the Dynamic Systems Lab (www.dynsyslab.org) at the University of Toronto Institute for Aerospace Studies (UTIAS), Canada
        {\tt\small k.pereidaperez@mail.uto\-ronto.ca} {\tt\small mohamed.helwa@robotics.utias.uto\-ronto.ca} {\tt\small schoellig@utias.uto\-ronto.ca}}%
\thanks{$^{2} $Mohamed K. Helwa is also with the Electrical Power and Machines Department, Cairo University, Egypt.}%
}

\newcommand\copyrighttext{\footnotesize \textbf{Sub version.} Accepted at \textit{2018 IEEE Robotics and Automation Letters.}

\textcopyright 2018 IEEE. Personal use of this material is permitted. Permission from IEEE must be obtained for all other uses, in any current or future media, including reprinting/republishing this material for advertising or promotional purposes, creating new collective works, for resale or redistribution to servers or lists, or reuse of any copyrighted component of this work in other works.}

\newcommand\copyrightnotice{\begin{tikzpicture}[remember picture,overlay]
\node[anchor=south,yshift=10pt] at (current page.south) {\fbox{\parbox{\dimexpr\textwidth-\fboxsep-\fboxrule\relax}{\copyrighttext}}};
\end{tikzpicture}}

\begin{document}
\maketitle

\copyrightnotice{} 

\begin{abstract}
Transfer learning has the potential to reduce the burden of data collection and to decrease the unavoidable risks of the training phase. In this paper, we introduce a multi-robot, multi-task transfer learning framework that allows a system to complete a task by learning from a few demonstrations of another task executed on another system. We focus on the trajectory tracking problem where each trajectory represents a different task, since many robotic tasks can be described as a trajectory tracking problem. The proposed, \emph{multi-robot} transfer learning framework is based on a combined $\lone$ adaptive control and iterative learning control approach. The key idea is that the adaptive controller forces dynamically different systems to behave as a specified reference model. The proposed \emph{multi-task} transfer learning framework uses theoretical control results (e.g., the concept of vector relative degree) to learn a map from desired trajectories to the inputs that make the system track these trajectories with high accuracy. This map is used to calculate the inputs for a new, unseen trajectory. Experimental results using two different quadrotor platforms and six different trajectories show that, on average, the proposed framework reduces the first-iteration tracking error by 74\% when information from tracking a different, single trajectory on a different quadrotor is utilized.
\end{abstract}

\section{INTRODUCTION}

This work aims to design a learning-based control method that guarantees high overall performance of robots carrying out complex tasks in changing environments. These control methods do not require the exact system model, and they are robust against model uncertainties, unknown disturbances, and changing dynamics. This is in contrast to traditional, model-based controllers, where small changes in the conditions may significantly deteriorate the controller performance and cause instability (see \cite{Skelton1989}, \cite{Morari1999} and \cite{Skogestad2007}). 

Training robots to operate in changing environments is complex and time-consuming. Transfer learning reduces the unavoidable risks of the training phase and the time needed to train each robot on each single task. In this paper, we develop a multi-robot, multi-task transfer learning framework that allows a system to complete a task by learning from a few demonstrations of another task executed on another system, see Fig.~\ref{fig:translearn}. We focus on trajectory tracking as many robotic tasks can be formulated as trajectory tracking problems. 
The proposed multi-robot, multi-task transfer framework achieves high-accuracy trajectory tracking from the first iteration for a wide range of robot dynamics and desired trajectories.
\begin{figure}[t]
   \centering
   \includegraphics[width=0.45\textwidth]{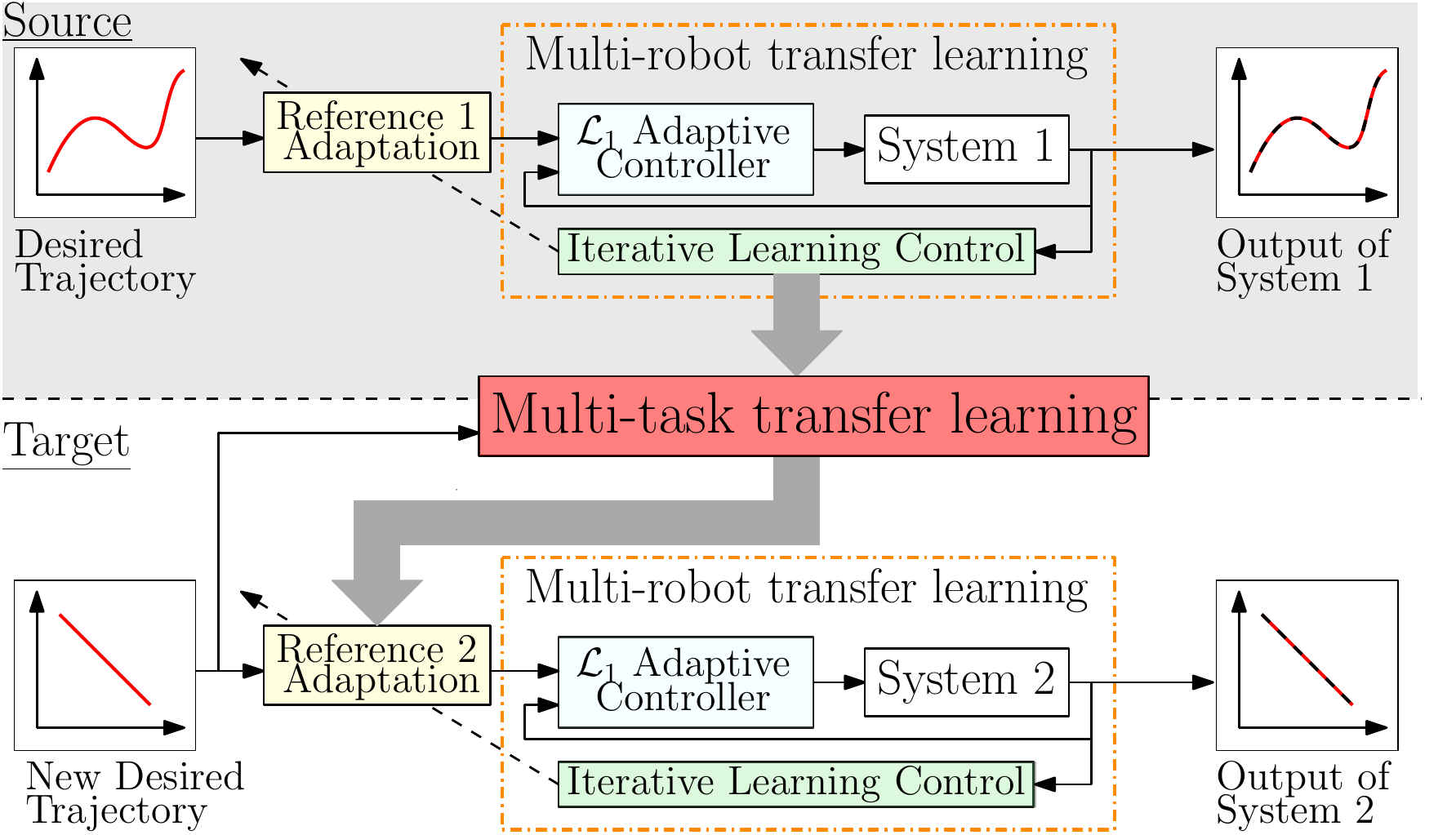}
   \caption{The proposed multi-robot, multi-task architecture. The \emph{multi-robot} transfer learning framework is composed of two methods \emph{(i)} an iterative learning control (ILC) module to learn an input such that the output tracks a desired output signal and \emph{(ii)} an $\lone$ adaptive controller to force different systems to behave in the same repeatable, predefined way. Hence, trajectories learned on a source system can directly be transferred to a target system. The \emph{multi-task} transfer learning framework learns a map from a desired trajectory to the inputs that make the system track it accurately. When a new trajectory is encountered, the learned map is used to calculate the inputs for the new trajectory.}
   \vspace{-14pt}
   \label{fig:translearn}
\end{figure}

In the realm of multi-robot transfer, the work in \cite{Raimalwala2015} learns an optimal static gain between the outputs of two linear, time-invariant (LTI), single-input, single-output (SISO) systems when they try to follow the same reference. The transformation error is minimized when the two systems have stable poles that lie close to each other. In~\cite{Raimalwala2016}, a preliminary study of transfer learning between two nonlinear, unicycle robots is presented. In~\cite{Helwa2017}, it is proved that the optimal transfer learning map between two robots is, in general, a dynamic system. An algorithm is also provided for determining the properties of the optimal dynamic map, including its order, relative degree and the variables it depends on. In \cite{Bocsi2013}, a data transfer mechanism based on manifold alignment of input-output data is proposed. The transferred data improves learning of a model of a robotic arm by using data from a different robotic arm. \rev{In~\cite{Harrison2017}, an algorithm for robust transfer of learned policies to target environments with unmodeled disturbances or unknown parameters is developed. This method improved performance on unmodeled disturbances by 50-300\%, compared to naive transfer. However, this algorithm was not evaluated on physical systems.}

Multi-task transfer is important to consider because learning approaches, such as iterative learning control (ILC), are usually not able to transfer knowledge from previously learned tasks to new, unseen tasks. In \cite{Hamer2013}, using knowledge from previously learned trajectories, a linear map is created to calculate the inputs required to track unseen trajectories. Experimental results show that only one learned trajectory is needed to improve the performance on a new trajectory. However, the variables required for the linear map are selected through experimental trial-and-error, which may be time consuming. \rev{In~\cite{Faust2013}, reinforcement learning is used to learn a value function for a particular load. Once learned, it is used to generate swing-free trajectories for aerial vehicles with suspended loads. Then, sufficient criteria are found to allow transfer of the learned policy to different models, states and action spaces.} In~\cite{Zhou2017} and~\cite{Li2017}, a deep neural network (DNN) is trained to achieve a unity map between the desired and actual outputs. The DNN adapts the reference signal to a feedback control loop to enhance the tracking performance of unseen trajectories. In \cite{Duan2017}, neural networks allow generalization of a task based on a single instance of the given task. However, the architecture of the neural network must be tailored to the specific task. 

In the realm of multi-task and multi-robot transfer, the work in \cite{Devin2017} trains neural network policies that can be decomposed into ``task-specific'' and ``robot-specific'' modules. When a new robot-task combination is encountered, the appropriate robot and task modules are composed to solve the problem. This architecture enables zero-shot generalization with a variety of robots and tasks in simulation. Neural network approaches require a significant amount of data and computational resources to train. In this paper, we emphasize data efficiency to achieve successful transfer in experiments. 

\rev{The goal of this work is to design a learning architecture that is able to achieve high-accuracy tracking in the first iteration 
 \begin{enumerate*}[font=\itshape]
 \item despite the presence of changing dynamics which include switching the robot hardware altogether, and 
 \item by using previously learned trajectories and generalizing knowledge to new, unseen trajectories. 
\end{enumerate*}}
We propose a multi-task transfer framework placed on top of a multi-robot transfer framework (see Fig.~\ref{fig:translearn}).  

The proposed, multi-robot transfer framework is based on the combined $\lone$ adaptive control and ILC approach introduced in \cite{Pereida2017}. \rev{The ILC improves tracking performance over iterations, while the $\lone$ adaptive controller forces dynamically-different nonlinear systems to behave as a specified linear model.} Hence, trajectories learned on a system can be transferred among dynamically different systems (equipped with the same underlying $\lone$ adaptive controller) to achieve high-accuracy tracking. The proposed, multi-task transfer scheme does not require exact knowledge of the system model or significant amounts of data as in~\cite{Devin2017}. Instead, it learns a map between a single desired trajectory and the inputs that make the system track the desired trajectory accurately. \rev{This map is used to calculate the inputs needed to track a new, unseen trajectory with high accuracy.} Unlike~\cite{Hamer2013}, where trial-and-error is used, we derive the structure of this map based on concepts from control theory, the vector relative degree. The proposed approach also allows the system to continue learning over iterations after transfer. 

The remainder of this paper is organized as follows. The problem is defined in Section~\ref{sec:problemstatement}. The details of the proposed approach and the main results are presented in Section~\ref{sec:methodology}. Experimental results on two quadrotors are presented in Section~\ref{sec:results}. Conclusions are provided in Section~\ref{sec:conclusions}.

\section{BACKGROUND}
\label{sec:background}
In this section, we review the definitions needed for the remainder of the paper. Consider a discrete-time, LTI, multi-input, multi-output (MIMO) system: 
\begin{equation}
\label{eq:mimosystem}
\begin{array}{rcl}
x(k+1) & = & A x(k) + B u(k) \\
y(k) & = & C x(k)\,,
\end{array}
\end{equation}
where $x(k)\in \mathbb{R} ^n$ is the system state vector, $u(k) = [u_1(k),\hdots,u_p(k)]^T\in \mathbb{R}^p$ is the system input, $y(k)=[y_1(k),\hdots,y_p(k)]^T\in\mathbb{R}^p$ is the system output, $k$ is the discrete-time index, and $B$ and $C$ are full rank. Let 
\[
 \begin{array}{lcr}
  B = [B_1,\hdots,B_p] & \text{and} & C=[C_1^T,\hdots,C_p^T]^T.
 \end{array}
\]
\begin{definition}[Vector relative degree] 
System~\eqref{eq:mimosystem} is said to have a vector relative degree $(r_1,\hdots,r_p)$ if 
\begin{enumerate}[font=\itshape]
 \item $C_iA^kB_j = 0$, $\forall i = \{1,\hdots,p\}$, $\forall j =\{ 1,\hdots,p\}$ and $\forall k =\{ 0,\hdots,r_{i-2}\}$, and
 \item the decoupling matrix $A_0$ where $[A_0]_{ij} = C_iA^{r_i-1}B_j$, $\forall i,j \in \{1,\hdots,p\}$ is nonsingular.
\end{enumerate}
\label{def:vrelativedegree}
\end{definition}
\begin{remark}
 If the system~\eqref{eq:mimosystem} has a vector relative degree $(r_1,\hdots,r_p)$, then we have 
 \begin{equation}
  \begin{bmatrix}
   y_1(k+r_1)\\
   \vdots \\
   y_p(k+r_p)
  \end{bmatrix} = 
  \begin{bmatrix}
   C_1A^{r_1} \\
   \vdots \\
   C_pA^{r_p}
  \end{bmatrix} x(k)+A_0
  \begin{bmatrix}
   u_1(k) \\
   \vdots \\
   u_p(k)
  \end{bmatrix}\,.
  \label{eq:relativedegree}
 \end{equation}
 \label{rem:relativedegree}
\end{remark}
For SISO systems, the relative degree is the number of sample delays between changing the input and observing a change of the output, and can be easily determined experimentally from the system's step response.
\begin{remark}
 Analogous to the case of SISO systems, the vector relative degree of a MIMO system can be determined from easy-to-execute experiments. In particular, one can carry out $p$ experiments; in each experiment, one should apply a unit step input to only one of the $p$ inputs, and monitor the responses of the $p$ outputs. From Definition~\ref{def:vrelativedegree}, a good estimate of the relative degree $r_i$ is the number of sample delays between changing \emph{any} of the inputs and seeing a change in the output $y_i$. 
 \label{rem:reldegreeexp}
\end{remark}
In other words, for the $p$ easy-to-execute experiments, one can observe the minimum time delay obtained from the different experiments for each output dimension. After estimating the vector relative degree $(r_1,\hdots,r_p)$, one still needs to verify that $A_0$ is invertible, see Definition~\ref{def:vrelativedegree}. Note that from Remark~\ref{rem:relativedegree}, the rank of $A_0$ can be indirectly checked by studying the rank of $Y_{r}$ where $[Y_{r}]_{ij}=y_i(r_i)|_j$
 and $y_i(r_i)|_j$ is the value of the output $y_i$ at time index $r_i$ for the step input experiment $j$, and assuming that for the $p$ experiments, the step inputs are applied at time index $0$ and the system is initiated from the same initial condition $x_0=0$.

We next review the definition of zero dynamics. The \emph{zero dynamics} of~\eqref{eq:mimosystem} are the invariant dynamics under which the system~\eqref{eq:mimosystem} evolves when the output $y$ is constrained to be zero for all times. The zero dynamics represent the internal dynamics of the system when the system output is stabilized to zero. By representing the system~\eqref{eq:mimosystem} in the so-called Byrnes-Isidori normal form, it can be shown that the order of the minimum realization of the zero dynamics is $n-r$, where $r=\sum_{i=1}^{p}r_i$ \cite{Isidori1995}, \cite{Henson1997}. The system~\eqref{eq:mimosystem} is \emph{minimum phase} if its zero dynamics are asymptotically stable in the Lyapunov sense. Based on the definition of zero dynamics and from~\eqref{eq:relativedegree}, it can be shown that if the system~\eqref{eq:mimosystem} is minimum phase, then the input $u(k)=-A_0^{-1}[(C_1A^{r_1})^T,\hdots,(C_pA^{r_p})^T]^T x(k)$ achieves asymptotic stability of the internal dynamics of~\eqref{eq:mimosystem}. We use this fact in Section~\ref{ssec:multitask}. 

Finally, we review the definition of the projection operator, which will be used in Section~\ref{ssec:multirobot}. We define $\nabla$ as the vector differential operator.
\begin{definition}[Projection operator]
 Consider a convex compact set with a smooth boundary given by $\Omega_c:=\{\lambda\in\mathbb{R}^n|f(\lambda)\leq c\}$, $0\leq c\leq1$, where $f:\mathbb{R}^n\rightarrow\mathbb{R}$ is the following smooth convex function:
 \[
 f(\lambda):=\frac{(\epsilon_\lambda +1)\lambda^T\lambda-\lambda_{max}^2}{\epsilon_\lambda \lambda_{max}^2}
 \]
 with $\lambda_{max}$ being the norm bound imposed on the vector $\lambda$, and $\epsilon_\lambda>0$ is a projection tolerance bound of our choice. The projection operator is defined as \cite{Hovakimyan2010}: 
 \[
  \normalfont \text{Proj}(\lambda,y):=
  \begin{cases}
   y & \text{if } f(\lambda)<0\\
   y & \text{if } f(\lambda)\geq 0 \text{ and } \nabla f^Ty\leq 0 \\
   y-f_\lambda & \text{if } f(\lambda)\geq 0 \text{ and } \nabla f^Ty> 0 \\
  \end{cases}
 \]
where $y\in \mathbb{R}^n$, $f_\lambda=\frac{\nabla f \nabla f^T y}{\Vert \nabla f \Vert^2} f(\lambda)$.
\end{definition}

\section{PROBLEM STATEMENT}
\label{sec:problemstatement}
\rev{The objective of this work is to achieve high-accuracy trajectory tracking in the first iteration in a multi-robot, multi-task framework, in which
\begin{enumerate*}[font=\itshape]
 \item there are uncertain, possibly changing environmental conditions, 
 \item the training and testing robots are dynamically different, and 
 \item the training and testing trajectories are different.
\end{enumerate*}}
We consider a control architecture as shown in Fig.~\ref{fig:translearn}. A multi-task framework is designed on top of a multi-robot framework. 

The multi-robot framework combines an adaptive controller with a learning-based controller as in~\cite{Pereida2017}. \rev{The adaptive controller forces two dynamically-different, minimum phase systems, which can be nonlinear systems, to behave in the same predefined way, as specified by a, possibly linear, reference model.} The learning-based controller improves the tracking performance over iterations. \rev{Since dynamically different systems are forced to behave in the same predefined way, trajectories learned on a training system can be directly applied to a different testing system achieving high-accuracy trajectory tracking in the first iteration.}

In the multi-task framework, information of the desired trajectory and the learned input trajectory $(\bff{y}{2}^{*,l},\bff{u}{2}^l)$ for a single, previously learned trajectory $l$ is given, where the subscript $2$ is used to agree with the multi-robot framework notation, see Fig.~\ref{fig:blockdiagram}. We aim to determine a map $\mathcal{M}$ from the desired outputs $\bff{y}{2}^{*,l}$ to the learned inputs $\bff{u}{2}^l$, such that the error between $\mathcal{M}\bff{y}{2}^{*,l}$ and $\bff{u}{2}^l$ is minimized, by solving: 
\begin{equation}
 \min_{\mathcal{M}} \Vert \mathcal{M}\bff{y}{2}^{*,l}-\bff{u}{2}^l \Vert\,.
 \label{eq:minerror}
\end{equation}
We show in this paper that for linearized or LTI systems the map $\mathcal{M}$ is time- and state-invariant, and consequently, when a new desired trajectory $(l+1)$ is encountered, an input trajectory that minimizes the trajectory tracking error can be calculated by: $\bff{u}{2}^{l+1} =\mathcal{M}\bff{y}{2}^{*,l+1}$.
Unlike~\cite{Hamer2013} where the structure of the map $\mathcal{M}$ is obtained through experimental trial-and-error, we use results from control systems theory to provide design guidelines for calculating the map $\mathcal{M}$. 


\section{METHODOLOGY}
\label{sec:methodology}
In this section, we describe in detail the multi-robot and multi-task transfer architectures for MIMO systems. 

\subsection{Multi-robot transfer}
\label{ssec:multirobot}

\rev{The multi-robot transfer framework is based on the idea that an adaptive controller is able to force two dynamically different systems to behave in a predefined way.} In this work we focus on MIMO systems. Therefore, any {\bf{MIMO $\lone$ adaptive controller}} implementation, such as \cite{Jafamejadsani2017}, \cite{Lee2017}, can be used. For convenience and completeness, we present the MIMO $\lone$ adaptive controller that we implemented in our experiments in Section~\ref{sec:results}. The extended $\lone$ adaptive controller assumes that the uncertain and changing dynamics of the robotic system can be described by a MIMO system for output feedback: 
\begin{equation}
 {y}_1(s) = A(s)({u}_{\lone}(s)+{d}_{\lone}(s))\,, \quad {y}_2(s)=\frac{1}{s}{y}_1(s)\,,
 \label{eq:mimooutput}
\end{equation}
where ${y}_1(s)$ and ${y}_2(s)$ are the Laplace transforms of the translational velocity ${y}_1(t)\in \mathbb{R}^p$ and the position ${y}_2(t)\in \mathbb{R}^p$, respectively, $A(s)$ is a transfer function matrix of strictly-proper \emph{unknown} transfer functions that can be stabilized by a proportional-integral controller, ${u}_{\lone}(s)$ is the Laplace transform of the input ${u}_{\lone}(t) \in \mathbb{R}^p$, and ${d}_\lone(s)$ is the Laplace transform of disturbance signals defined as: $d_{\lone}(t):= f(t,{y}_1(t))$, where $f:\mathbb{R}\times\mathbb{R}^p\rightarrow\mathbb{R}^p$ is an \emph{unknown} map subject to the global Lipschitz continuity assumption with Lipschitz constant $L$ (see Assumption 4.1.1 in~\cite{Hovakimyan2010}). 

\begin{figure}[t]
   \centering
    \vspace{7pt}
   \includegraphics[width=0.45\textwidth]{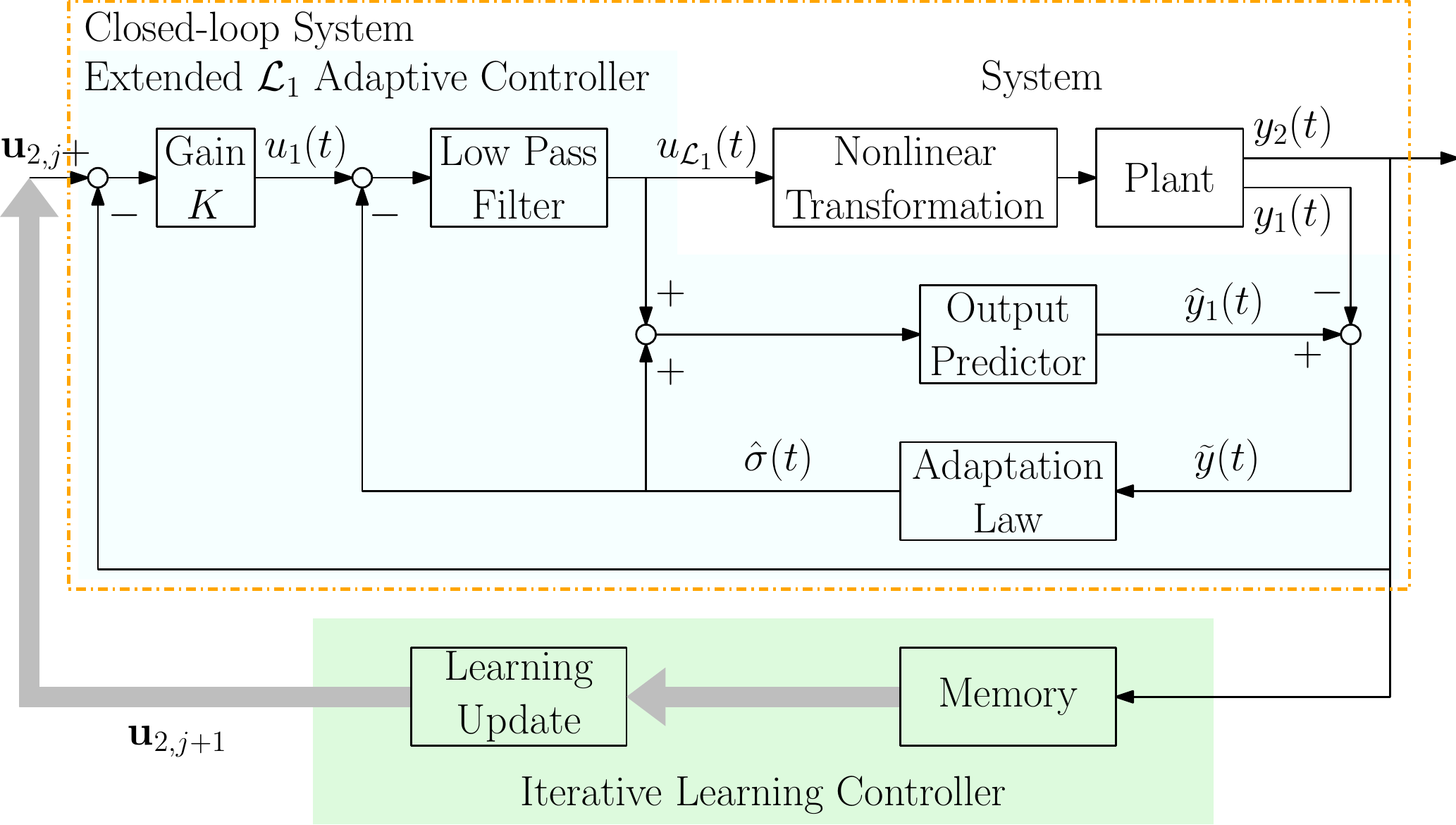}
   \caption{The multi-robot transfer learning architecture combining the extended $\lone$ adaptive control and iterative learning control. Figure adopted from~\cite{Pereida2017}. The thick, light grey arrows represent signals of a whole trajectory execution.}
   \label{fig:blockdiagram}
   \vspace{-14pt}
\end{figure}

The $\lone$ adaptive output feedback controller aims to design a control input ${u}_{\lone}(t)$ such that the output ${y}_2(t)\in\mathbb{R}^p$ tracks a bounded piecewise continuous reference input $u_2(t)\in \mathbb{R}^p$. We aim to achieve a desired closed-loop behavior by nesting the output of the $\lone$ adaptive controller with output ${y}_1(t)\in\mathbb{R}^p$ which tracks ${u}_1(t)\in \mathbb{R}^p$ within a proportional feedback loop (see Fig.~\ref{fig:blockdiagram}). The desired inner loop behavior is given by the following first-order reference dynamic systems: 
\begin{equation}
 M(s) = \text{diag} (M_1(s),\hdots,M_p(s)), \ M_i(s)=\frac{m_i}{s+m_i}\,, 
 \label{eq:refsys}
\end{equation}
where $m_i>0$. We can rewrite system~\eqref{eq:mimooutput} in terms of the reference system~\eqref{eq:refsys} as follows:
\begin{equation}
 {y}_1(s) = M(s)({u}_{\lone}(s)+{\sigma}(s))\,, \quad y_2(s) = \frac{1}{s}y_1(s)\,,
\end{equation}
where the uncertainties in $A(s)$ and ${d}_\lone(s)$ are combined into ${\sigma}(s)$ \rev{, given by
$\sigma (s)=M^{-1}(s) ((A(s)-M(s))u_\lone(s)+A(s)d_\lone(s))$\,.} The equations that describe the implementation of the extended $\lone$ output feedback controller are: 
\begin{description}
 \item [Output Predictor:] $\widehat{{y}}_1(s) = M(s)({u}_\lone (s) +\widehat{{\sigma}}(s))\,,$ where $\widehat{{\sigma}}(s)$ is the adaptive estimate of ${\sigma}(s)$.
 \item [Adaptation Law:] $\dot{\widehat{{\sigma}}}(t) = \Gamma \text{Proj}(\widehat{{\sigma}}(t), -\widetilde{y}(t))\,,$
 with $\widehat{{\sigma}}(0) = 0$, where $\widetilde{y}(t):=\widehat{{y}}_1(t)-{y}_1(t)$ and $\text{Proj(. , .)}$ is defined in Section~\ref{sec:background}. The adaptation rate $\Gamma >0$ is subject to the lower bound specified in \cite{Hovakimyan2010} and is set very large for fast adaptation. 
 \item [Control Law:] ${u}_\lone(s) = V(s)({u}_1(s)-\hat{\sigma}(s))\,,$ where $V(s)$ is a diagonal transfer function matrix of low-pass filters: 
 \[
   V(s) = \text{diag}(V_1(s),\hdots,V_p(s))\,, \quad V_i(s) = \frac{\omega_i}{s+\omega_i}\,, 
 \]
where $\omega_i>0$. By filtering out the high frequencies in $\hat{\sigma}(s)$, high adaptation gains can be used without deteriorating the transient behavior of the system. 
\item [Closed-Loop Feedback:] 
\begin{equation} 
 {u}_1(s)=K(u_2(s)-{y}_2(s))\,,
 \label{eq:u1K}
\end{equation} 
where $K= \text{diag}(K_1,\hdots,K_p)$, and $K_i\in\mathbb{R}^+$ is a proportional gain; used to make ${y}_2(t)$ track ${u}_2(t)$.
\end{description}

We design $V(s)$, $K$, and $M(s)$ such that $H(s)=A(s)(I-V(s)+V(s)M^{-1}(s)A(s))^{-1}$ and $F(s)=(sI+H(s)V(s)K)^{-1}$ are stable, and the following $\lone$-norm condition is satisfied $\Vert F(s)H(s)(1-V(s)) \Vert_\lone L < 1$, where $L$ is the global Lipschitz constant of the disturbance function $f$.

The extended $\lone$ adaptive controller makes the system (roughly) behave as a linear, MIMO system described by: 
\begin{equation}
\begin{array}{c}
 {y}_2(s) = \text{diag}(D_1(s),\hdots,D_p(s)){u}_2(s)\,,\,\, \text{where} \\
D_i(s) = \frac{K_im_i}{s^2+m_is+K_im_i}\,.
\end{array}
\label{eq:clrefsys}
\end{equation}

\rev{The multi-task transfer framework, discussed in the next subsection, requires a desired trajectory and the correct input that makes the system track the desired trajectory. To construct this pair of desired trajectory and corresponding correct input, we use an {\bf{optimization-based ILC}}~\cite{Schoellig2009} to modify the input and improve the tracking performance of the system, which now behaves close to~\eqref{eq:clrefsys}, in a small number of iterations $1,\hdots,j$.} The objective is to make ${y}_2(t)$ track a desired output trajectory ${{y}}_2^*(t)$, which is defined over a finite-time interval. We assume that there exist input, state and output trajectories $({{u}}_2^*(t), {{x}}^*(t), {{y}}_2^*(t))$ that are feasible with respect to the true dynamics of the $\lone$-controlled system under linear constraints on the system inputs and/or outputs. We also assume that the system stays close to the desired trajectory; hence, we only consider small deviations from the above trajectories $(\widetilde{{u}}_2(t), \widetilde{{x}}(t), \widetilde{{y}}_2(t))$. The small deviation signals are discretized since the input of computer-controlled systems is sampled and measurements are only available at fixed time intervals. \rev{We use Equation (12) in~\cite{Schoellig2012} to write the output and input trajectories in the lifted representation as $\tbff{y}{2} = [\widetilde{y}_{2}^{T}(1),\hdots, \widetilde{y}_{2}^{T}(N)]^T$, where $N<\infty$ is the number of discrete samples, $\widetilde{y}_{2}(k) = [\widetilde{y}_{2,1}(k), \hdots, \widetilde{y}_{2,p}(k)]^T\in \mathbb{R}^p$ and $\widetilde{\bf{u}}_2=[\widetilde{u}_2^T(0),\hdots,\widetilde{u}_2^T(N-1)]^T$ where $\widetilde{u}_2(k) = [\widetilde{u}_{2,1}(k), \hdots, \widetilde{u}_{2,p}(k)]^T\in \mathbb{R}^p$. }
 
Now consider the minimum phase, discrete-time, LTI, MIMO system: 
\begin{equation}
\begin{array}{c}
 \widetilde{x}(k+1) = A_\lone\widetilde{x}(k)+B_\lone\widetilde{u}_2(k)\\
 \widetilde{y}_2(k) = C_\lone \widetilde{x}(k)\,,
 \end{array}
 \label{eq:discretemodel}
\end{equation}
where $A_\lone$, $B_\lone$ and $C_\lone$ are the discrete-time matrices that describe~\eqref{eq:clrefsys}. \rev{To capture the dynamics of a complete trial by a static mapping, we compute the lifted representation of~\eqref{eq:discretemodel} using Equations~(13-14) in~\cite{Schoellig2012} to obtain:
\begin{equation}
 \tbff{y}{2,j}=\tbff{F}{\text{ILC}}\tbff{u}{2,j} + \tbff{d}{\infty}\,
 \label{eq:fuplusd}
\end{equation}
where the subscript $j$ represents the iteration number, $\tbff{F}{\text{ILC}}$ is a constant matrix derived from the discretized model~\eqref{eq:discretemodel} and $\tbff{d}{\infty}$ represents a repetitive disturbance that is initially unknown, but is identified during the learning process.} In many control applications, constraints must be placed on the process variables to ensure safe and smooth operations. The system may be subject to linear input or output constraints: 
\begin{equation}
 S_c\tbff{y}{2}\leq \tbff{y}{2,max}\,,\quad Z_c\tbff{u}{2}\leq \tbff{u}{2,max}\,,
 \label{eq:constraints}
\end{equation}
where $S_c$ and $Z_c$ are matrices of appropriate size. 

\rev{Following Equations (19-26) in \cite{Schoellig2012}, we implement an iteration-domain Kalman filter to obtain an estimate of the disturbance $\tbff{d}{\infty}$, $\hbff{d}{j|j}$, based on measurements from iterations $1,\hdots,j$.} Using~\eqref{eq:fuplusd} and the estimated disturbance, the estimated output error $\hbff{y}{2,j+1|j}$ can be represented by: 
\begin{equation}
 \hbff{y}{2,j+1|j} =\tbff{F}{\text{ILC}}\tbff{u}{2,j+1} + \hbff{d}{j|j}\,.
 \label{eq:Kalmanestimation}
\end{equation}
At the end of iteration $j$, an update step computes the next input sequence $\tbff{u}{2,j+1}$ that minimizes the estimated output error $\hbff{y}{2,j+1|j}$ and the control effort $\tbff{u}{2,j+1}$ based on the following cost function: 
\[
 J(\tbff{u}{2,j+1}) = \hbff{y}{2,j+1|j}\bff{Q}{}\hbff{y}{2,j+1|j}+\tbff{u}{2,j+1}\bff{R}{}\tbff{u}{2,j+1}\,,
\]
subject to~\eqref{eq:constraints}, where $\bff{Q}{}$ is a constant, positive semi-definite matrix and $\bff{R}{}$ is a constant, positive definite matrix. The resulting convex optimization problem can be solved efficiently with state-of-the-art optimization libraries. In this work, we use the IBM ILOG CPLEX Optimizer.

\begin{remark}
\label{rem:diffrefsys}
If the source and target systems have underlying $\lone$ adaptive controllers with different reference models, then it is still possible to implement the multi-robot framework by using the reference models to build a map from the source system to the target system \cite{Helwa2017}. Using this map, trajectories learned on the source system can be transferred to the target system, which has a different reference model. 
\end{remark}

\subsection{Multi-task transfer}
\label{ssec:multitask}
\rev{To facilitate transfer learning between different desired trajectories, a mapping is learned from desired trajectories to the inputs that make the system track these trajectories.} Our proposed multi-task transfer learning framework uses insights from control systems theory and knowledge of a single previously learned trajectory to calculate the reference input of a new, unseen trajectory to minimize the trajectory tracking error. We showed that the input-output behavior of the system under the extended $\lone$ adaptive control is LTI, MIMO and minimum phase, see~\eqref{eq:clrefsys}. We use~\eqref{eq:mimosystem} to describe the controlled system and assume it is minimum phase. 

\begin{lemma}
 Consider a minimum phase, discrete-time, MIMO, LTI system~\eqref{eq:mimosystem}, and a smooth desired trajectory $\bff{y}{2}^*$. Then, there exists a control input sequence $\bff{u}{2}$ that achieves perfect tracking of $\bff{y}{2}^{*}$. Moreover, at time instant $k$, the control input $u_2(k)$ can be represented as a linear combination of the state $x(k)$ and the values ${y}_{2,1}^{*}(k+r_1), \hdots, {y}_{2,p}^{*}(k+r_p)$, where $y^*_{2,i}(j)$ is the value of the $i^\text{th}$ component of the desired output at time index $j$, and $(r_1,\hdots,r_p)$ is the vector relative degree of the system.
 \label{lem:linearcomb}
\end{lemma}
\begin{proof}
 From Definition~\ref{def:vrelativedegree} the matrix $A_0$ is nonsingular. Hence, we can define the following control law: 
 \begin{equation}
  \begin{bmatrix}
   u_{2,1}(k) \\
   \vdots \\
   u_{2,p}(k)
  \end{bmatrix}= A_0^{-1} \left(
  \begin{bmatrix}
   -C_1A^{r_1} \\
   \vdots \\
   -C_pA^{r_p}
  \end{bmatrix} x(k) +
  \begin{bmatrix}
   y_{2,1}^*(k+r_1) \\
   \vdots \\
   y_{2,p}^*(k+r_p)
  \end{bmatrix} 
\right)\,.
\label{eq:uequation}
 \end{equation}
By assumption, the desired trajectory ${y}_{2}^{*}(k)\,,\,k\geq 0$, is known. Hence, at time step $k$, $y_{2,1}^*(k+r_1),\hdots,y_{2,p}^*(k+r_p)$ are all known, and~\eqref{eq:uequation} can be realized. 

Next, by plugging the control law~\eqref{eq:uequation} into~\eqref{eq:relativedegree}, we obtain: 
\begin{equation}
 \begin{bmatrix}
  y_1(k+r_1) \\
  \vdots \\
  y_p(k+r_p)
 \end{bmatrix} =
 \begin{bmatrix}
  y_{2,1}^*(k+r_1) \\
  \vdots \\
  y_{2,p}^*(k+r_p)
 \end{bmatrix}\,.
 \label{eq:equalout}
\end{equation}

Since the system~\eqref{eq:mimosystem} is minimum phase by assumption, then as discussed in Section~\ref{sec:background}, the control input $u_2(k)=-A_0^{-1}[(C_1A^{r_1})^T,\hdots,(C_pA^{r_p})^T]^T x(k)$ achieves asymptotic stability of the internal dynamics of~\eqref{eq:mimosystem}. Since for linear systems asymptotic stability implies bounded-input, bounded-output (BIBO) stability, then it is evident that under the control law~\eqref{eq:uequation} and for bounded reference input $y_2^*$, the internal states of~\eqref{eq:mimosystem} are bounded.

Finally, by applying the control law~\eqref{eq:uequation} at each time step $k$, $k\geq 0$, we have from~\eqref{eq:equalout}, $y_i(\widetilde{k})=y_{2,i}^*(\widetilde{k})$, for each $\widetilde{k}\geq r_i$, i.e., the perfect tracking condition is achieved. From the control law~\eqref{eq:uequation}, the input $u_2(k)$ is a linear combination of $x(k)$ and $y_{2,1}^*(k+r_1), \hdots, y_{2,p}^*(k+r_p)$, which completes the proof.
\end{proof}

\rev{Suppose that we are given a smooth desired trajectory $\bff{y}{2}^*$ and using a learning approach, such as ILC, we are able to obtain for this particular desired trajectory an input sequence vector $\bff{u}{2}$ that achieves high-accuracy tracking performance.} Similar to \cite{Hamer2013}, we use the desired output and learned input trajectories $(\bff{y}{2}^*, \bff{u}{2})$ to learn a transfer learning map $\mathcal{M}$ that minimizes the transfer learning error as in~\eqref{eq:minerror}. Unlike~\cite{Hamer2013}, in which the map is constructed through experimental trial-and-error, we use Lemma~\ref{lem:linearcomb} to construct the map. In particular, we know from Lemma~\ref{lem:linearcomb} that to achieve perfect tracking, $u_{2,i}(k)$, $i=1,\hdots,p$, should be a linear combination of $x(k)$ and $y_{2,1}^*(k+r_1),\hdots,y_{2,p}^*(k+r_p)$, where $(r_1,\hdots,r_p)$ is the vector relative degree of~\eqref{eq:mimosystem}. We assume, for now, that the state $x(k)$ can be measured or
estimated, and stored. Hence, we propose to build, with the available information, the following windowing function: 
\begin{equation}
{W}(\bff{x}{},\bff{y}{2}^*) = 
 \begin{bmatrix}
  x^T(0) & \bar{y}_{2}^*(0) \\
  \vdots & \vdots  \\
  x^T(N_{r}) & \bar{y}_{2}^*(N_{r})  
 \end{bmatrix}\,,
 \label{eq:windowing}
\end{equation}
where $\bar{y}_{2}^*(a) = [y_{2,1}^*(a+r_1),\hdots,y_{2,p}^*(a+r_p)]$, and $N_{r}=N-\max_{j\in \{1,\hdots,p \}} (r_j)$. Using the windowing function ${W}(\bff{x}{},\bff{y}{2}^*)$, we define the following learning process: 
\begin{equation}
 \bff{u}{2,i}={W}(\bff{x}{},\bff{y}{2}^*)\theta_i\,,
 \label{eq:theta}
\end{equation}
where $\bff{u}{2,i}=[u_{2,i}(0),\hdots,u_{2,i}(N_{r})]^T$ is the collection of the $i^{\text{th}}$ elements of $\bff{u}{2}$, obtained from the ILC algorithm. \rev{This is a linear regression problem for the parameter vector $\theta_i\in\mathbb{R}^{n+p}$. }
\begin{remark}
 The vectors of unknowns $\theta_i$, $i\in \{1,\hdots,p\}$, are all functions of the system matrices $A$, $B$, $C$, which is evident from~\eqref{eq:uequation} and the definition of the matrix $A_0$ in Definition~\ref{def:vrelativedegree}. We emphasize that the vectors $\theta_i$ do not depend on the desired trajectory $\bff{y}{2}^*$ or the system states. 
\end{remark}
Therefore, we can reuse the calculated vectors $\theta_i$, which build an invariant map, to calculate for new, unseen, desired trajectories correct input vectors that achieve perfect tracking. In particular, we use the vectors $\theta_i$ to calculate the control input that achieves perfect tracking of a new desired trajectory $\bff{y}{2}^{*,new}$ as follows: 
 \begin{equation}
  {u}_{2,i}^{new}(k)= \begin{bmatrix}
                        x^T(k) & \bar{y}_{2}^{*,new}(k)
                       \end{bmatrix}
\theta_i\, \quad \forall i\in \{1,\hdots,p\}\,,
  \label{eq:unew}
 \end{equation}
where $\theta_i$, $i\in \{1,\hdots,p\}$ are calculated by~\eqref{eq:theta} and $\bar{y}_{2}^{*,new}(k) = [y_{2,1}^{*,new}(k+r_1),\hdots,y_{2,p}^{*,new}(k+r_p)]$. 
\begin{remark}
 Our proposed control law~\eqref{eq:unew} does not assume the knowledge of the system matrices $A$, $B$ and $C$ as in~\eqref{eq:uequation}. It only assumes the knowledge of the vector relative degree of the system, which can be calculated from~\eqref{eq:clrefsys}, or obtained through experiments, see Remark~\ref{rem:reldegreeexp}.
\end{remark}

Notice that the construction of the windowing function $W(\bff{x}{},\bff{y}{2}^*)$ requires the knowledge of the system states or estimated values of the states. We now discuss how the proposed transfer learning approach can be extended to the case where the state measurements or their estimated values are not available. To this end, we review an important Lemma from \cite{Lewis2011}. 

\begin{lemma}
 Consider the discrete-time, LTI system~\eqref{eq:mimosystem}, and suppose that the pair $(A,C)$ is observable. Then, the system state $x(k)$ is given uniquely in terms of input-output sequences as follows: 
 \begin{equation}
  x(k) = M_u 
  \begin{bmatrix}
   u(k-1) \\
   \vdots \\
   u(k-\bar{N})
  \end{bmatrix}+ M_y
  \begin{bmatrix}
   y(k-1)\\
   \vdots \\
   y(k-\bar{N})
  \end{bmatrix}\,,
 \end{equation}
where $\bar{N}$ is an upper bound on the system's observability index, $M_u = U_{\bar{N}}-M_yT_{\bar{N}}$, $M_y = A^{\bar{N}}V_{\bar{N}}^+$, $V_{\bar{N}}^+ = (V_{\bar{N}}^TV_{\bar{N}})^{-1}V_{\bar{N}}^T$ is the left inverse of the observability matrix $V_{\bar{N}}$, and 
\[
 U_{\bar{N}} = \begin{bmatrix} B & AB & \hdots & A^{\bar{N}-1}B \end{bmatrix}\,,
\]
\[
 T_{\bar{N}} = 
 \begin{bmatrix}
  0 & CB & CAB & \hdots & CA^{\bar{N}-2}B \\
  0 & 0 & CB & \hdots & CA^{\bar{N}-3}B \\
  \vdots & \vdots & \ddots  & \ddots & \vdots \\
  0 & \hdots & & 0 & CB \\
  0 & 0 & 0 & 0 & 0 
 \end{bmatrix}\, ,\, 
\]
\[
 V_{\bar{N}} = 
 \begin{bmatrix}
  (CA^{\bar{N}-1})^T & 
  \hdots & 
  (CA)^T &
  C^T
 \end{bmatrix}^T\,.
\]
\label{lem:statereconst}
\end{lemma}
\vspace{-12pt}
In other words, if the system~\eqref{eq:mimosystem} is observable, then the state vector can be represented as a linear combination of a finite sequence of past inputs and outputs of the system. From Lemmas~\ref{lem:linearcomb} and~\ref{lem:statereconst}, it can be shown that the control input $u_2(k)$ that achieves perfect tracking of the desired trajectory can be represented as a linear combination of $u_2(k-1),\hdots,u_2(k-\bar{N}),y(k-1),\hdots,y(k-\bar{N}), y_{2,1}^*(k+r_1),\hdots,y_{r,p}^*(k+r_p)$, where $y(k)$ is the actual output of the system at time step $k$, $(r_1,\hdots,r_p)$ is the vector relative degree of~\eqref{eq:mimosystem} and $\bar{N}$ is an upper bound of the observability index of~\eqref{eq:mimosystem}.

Hence, for observable systems, our approach is still valid when only input-output data is available. In this case, the ILC algorithm will also be used to calculate an input $\bff{u}{}$ that will make the output $\bff{y}{}$ track the desired trajectory $\bff{y}{2}^*$ with high accuracy. Using this input-output information, we can redefine the windowing function~\eqref{eq:windowing} by substituting the state $x(k)$ with a finite sequence of input-output information such that: 
\begin{equation}
 W_{IO}(\bff{u}{},\bff{y}{},\bff{y}{2}^*)= 
 \begin{bmatrix}
  \bar{u}(\bar{N}) & \bar{y}(\bar{N}) & \bar{y}_{2}^*(\bar{N}) \\
  \vdots & \vdots & \vdots \\
  \bar{u}(N_r) & \bar{y}(N_r) & \bar{y}_{2}^*(N_r) \\
 \end{bmatrix}\,,
\end{equation}
where $\bar{y}_{2}^*(a) = [y_{2,1}^*(a+r_1),\hdots,y_{2,p}^*(a+r_p)]$, $\bar{u}(a) = [u^T(a-1),\hdots,u^T(a-\bar{N})]$, and $\bar{y}(a) = [y^T(a-1),\hdots,y^T(a-\bar{N})]$. Similar to~\eqref{eq:theta}, we then calculate the vectors of unknown parameters $\theta_{IO,i}$. The vectors are functions of the matrices of system~\eqref{eq:mimosystem} and can be used to calculate the required input when a new, unseen trajectory is encountered. For this case, the proposed control law is: 
\[
 u_{2,i}^{new}(k) = 
 \begin{bmatrix}
  \bar{u}(k) & \bar{y}(k) & \bar{y}_{2}^{*,new}(k)
 \end{bmatrix}
 \theta_{IO,i}\,,\ \forall i \in \{1,\hdots,p\}\,.
\]

In order to allow the system to continue learning after transfer, we need to provide the ILC with an initial estimate of the repetitive disturbance $\hbff{d}{0,trans}$. We assume that the calculated input $\bff{u}{2}^{new}$ achieves perfect tracking of the new trajectory; hence, $\hbff{y}{2,j+1|j}$ in~\eqref{eq:Kalmanestimation} is 0. Using~\eqref{eq:Kalmanestimation} and the calculated input $\bff{u}{2}^{new}$, we are able to calculate the disturbance $\hbff{d}{0,trans}$. 

\rev{In this subsection, we derived the \emph{multi-task} transfer learning framework for linear/linearized systems. This is not a big restriction given that the underlying $\lone$ adaptive controller is able to force a nonlinear system to approximately behave as a linear model.} \rev{In fact, the $\lone$ adaptive controller can be replaced by another adaptive controller with a reference model, e.g., a model reference adaptive controller. However, we implement the $\lone$ adaptive controller as it provides performance bounds on the output and input signals~\cite{Hovakimyan2010}.} Nevertheless, it should be noted that the proposed framework can be also extended to {\bf{nonlinear systems}} with well-defined vector relative degrees and stable inverse dynamics. In particular, analogous to Lemma~\ref{lem:linearcomb}, it can be shown that there exists a control input satisfying perfect tracking of an arbitrary, smooth trajectory ${\bf{y}}_2^*$, and this input is a nonlinear function of the state $x(k)$ and the values $y_{2,1}^*(k+r_1),\cdots,y_{2,p}^*(k+r_p)$, where $(r_1,\hdots,r_p)$ is the system's vector relative degree. In practice, one can use a nonlinear regression model, e.g., a neural network or a Gaussian Process, to approximate this unknown nonlinear function \cite{Zhou2017}, \cite {Li2017}. Alternatively, smooth nonlinear systems can be approximated by piecewise affine/linear systems with arbitrary accuracy \cite{Helwa2015}. One can construct a cover/partition of the state space of a nonlinear system, and represent the nonlinear system in each region of the cover/partition with a local, affine/linear model. One can use the above proposed results to define a linear transfer learning map for each local region, resulting overall in a piecewise linear transfer learning map. The details will be further studied in future publications.

\section{EXPERIMENTAL RESULTS}
\label{sec:results}
This section shows the experimental results of the proposed multi-robot, multi-task framework applied to quadrotors for high-accuracy trajectory tracking. \rev{We assess three aspects to verify the effectiveness of the proposed framework: 
\begin{enumerate*}[font=\itshape]
 \item capability to transfer different trajectories between different quadrotor platforms, 
 \item repeatability of results, and 
 \item use of different reference models for the $\lone$ adaptive controllers of the source and target systems.
\end{enumerate*} }
These experiments significantly extend the experiments in~\cite{Pereida2017}, which only assessed the multi-robot transfer learning framework.

The vehicles used in the experiments are the Parrot AR.Drone 2.0 and the Parrot Bebop 2 (see Fig.~\ref{fig:drones}). Each quadrotor has an underlying $\lone$ adaptive controller that makes both vehicles behave close to a reference system. \rev{In what follows, the signals $u_{1,i}(t)$ are the desired translational velocities, $u_{2,i}(t)$ are the desired translational positions, $y_{1,i}(t)$ are the quadrotor translational velocities and $y_{2,i}(t)$ are the quadrotor translational positions in the $i = {x,y,z}$ directions, respectively.} A central overhead motion capture camera system provides position, roll-pitch-yaw Euler angles and rotational velocity measurements, and through numerical differentiation, \rev{we obtain translational velocities}. \rev{We propose six different trajectories to test our approach, as shown in Fig.~\ref{fig:desiredtrajectories3D}.} To quantify the tracking performance, an average position error along the trajectory is defined by: \rev{
\begin{equation}
 e = \frac{1}{N}\sum_{i=1}^N  \sqrt{e_x^2(i)+e^2_y(i)+e^2_z(i)} \,,
\end{equation} }
where $e_j(i)=u_{2,j}(i)-y_{2,j}(i)$ and $j=x,y,z$. 

\begin{figure}[t]
   \centering
   \vspace{5pt}
   \includegraphics[width=0.35\textwidth]{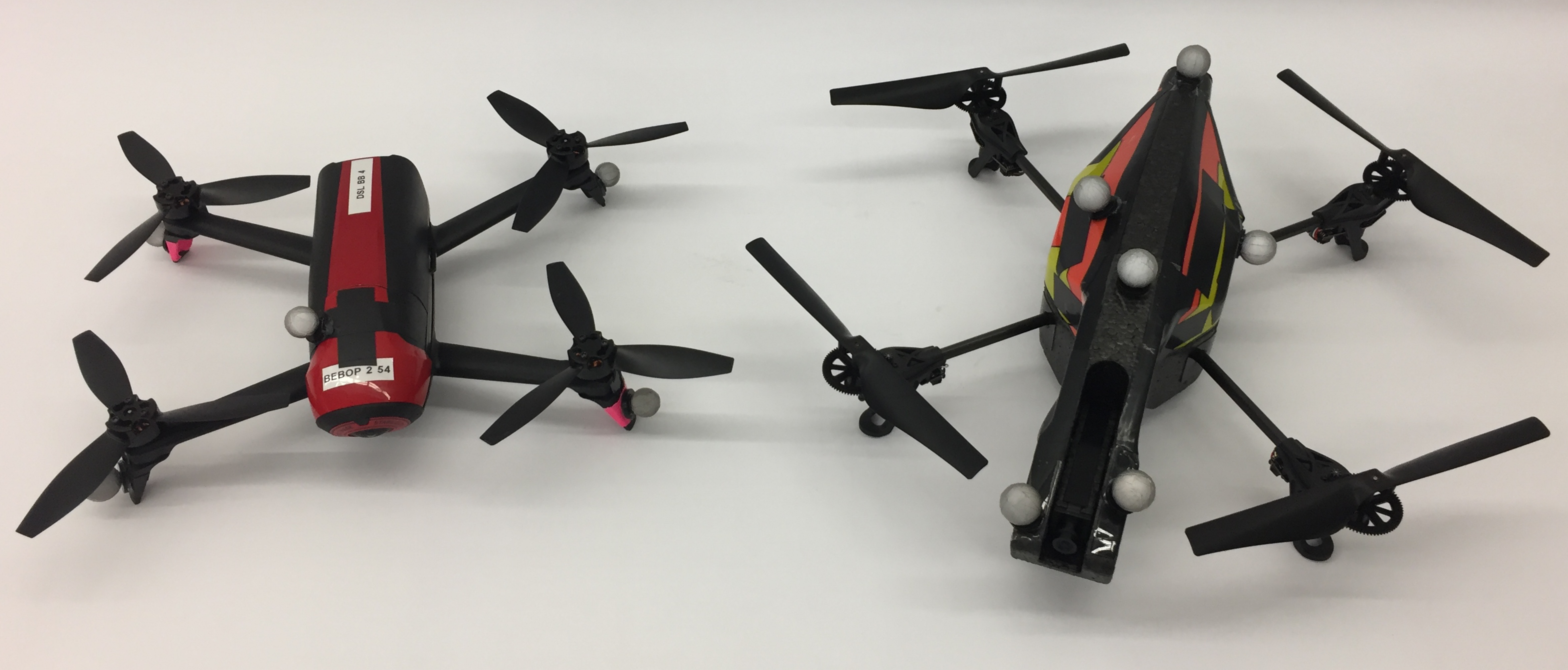}
   \vspace{-7pt}
   \caption{The two different quadrotors used in the experiments. The quadrotor on the left is the Bebop 2, and on the right is the AR.Drone 2.0.}
   \label{fig:drones}
   \vspace{-11pt}
\end{figure}
\begin{figure}[t]
   \centering
   \includegraphics[width=0.33\textwidth]{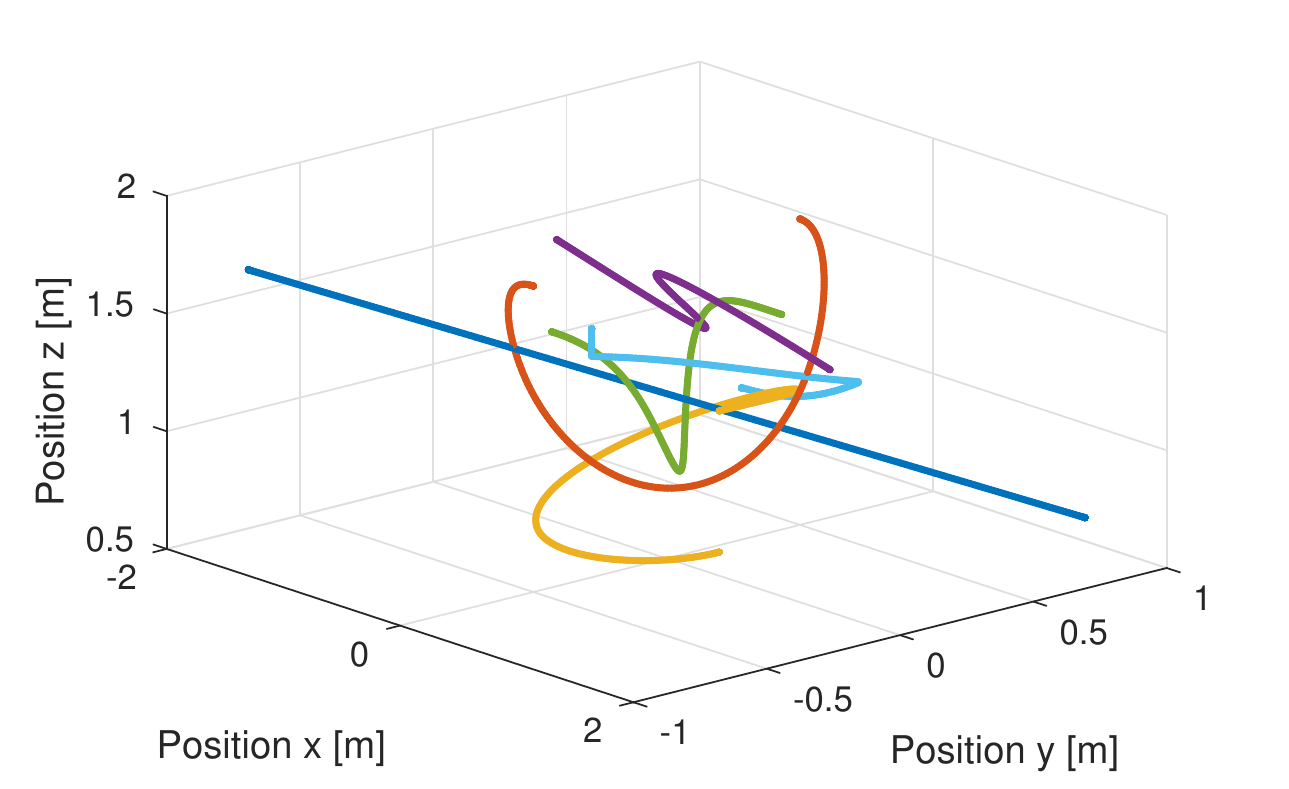}
   \vspace{-8pt}
   \caption{\rev{The six different trajectories that are used to test the multi-robot, multi-task transfer framework. The source and target trajectories used to assess repeatability and different reference models are depicted in dark blue and orange, respectively.} }
   \label{fig:desiredtrajectories3D}
   \vspace{-11pt}
\end{figure}

\subsection{\rev{Multi-robot, multi-task transfer: capabilities}}
\rev{To assess the performance of the proposed multi-robot, multi-task transfer framework under different conditions, we generate six different trajectories (see Fig.~\ref{fig:desiredtrajectories3D}) to be transferred from the AR.Drone~2.0 to the Bebop~2. We learn each of the six trajectories on the AR.Drone~2.0 using the multi-robot framework (Section~\ref{ssec:multirobot}). We devise a \emph{one-to-all} transfer scheme, in which we use one of the six trajectories as a source trajectory to transfer to each of the six trajectories, now called target trajectories (including the source trajectory, which is the the multi-robot transfer learning case). Using the multi-task framework (Section~\ref{ssec:multitask}), we apply the one-to-all scheme six times such that each of the six trajectories is the source trajectory once (6 source $\times$ 6 target trajectories). Fig.~\ref{fig:all_errors1} shows, for each of the six target trajectories, the tracking error on a Bebop~2 during a learning process at iteration 1 (open circles) and at iteration 10 (filled circles) when no transfer information is used. The tracking errors in the first iteration after applying the one-to-all, multi-robot, multi-task transfer are summarized in six boxplots, each of which corresponds to one target trajectory. The red mark on each box indicates the median, while the bottom and top edges of the box indicate the $25^{th}$ and $75^{th}$ percentiles, respectively. The whiskers represent the most extreme data points. Using the proposed framework significantly reduces the error in the first iteration. On average, the percentage of error reduction for the 36 experiments is 74.12\%, when using \emph{only 6 seconds} of training data (length of each trajectory). This is a significant improvement compared to 
 \begin{enumerate*}[font=\itshape]
  \item \cite{Hamer2013} where the structure of the transfer map is determined by trial-and-error, and 
  \item \cite{Devin2017} where the training of neural networks requires a significant amount of data. 
 \end{enumerate*}}

\begin{figure}[t]
   \centering
   \includegraphics[width=0.4\textwidth]{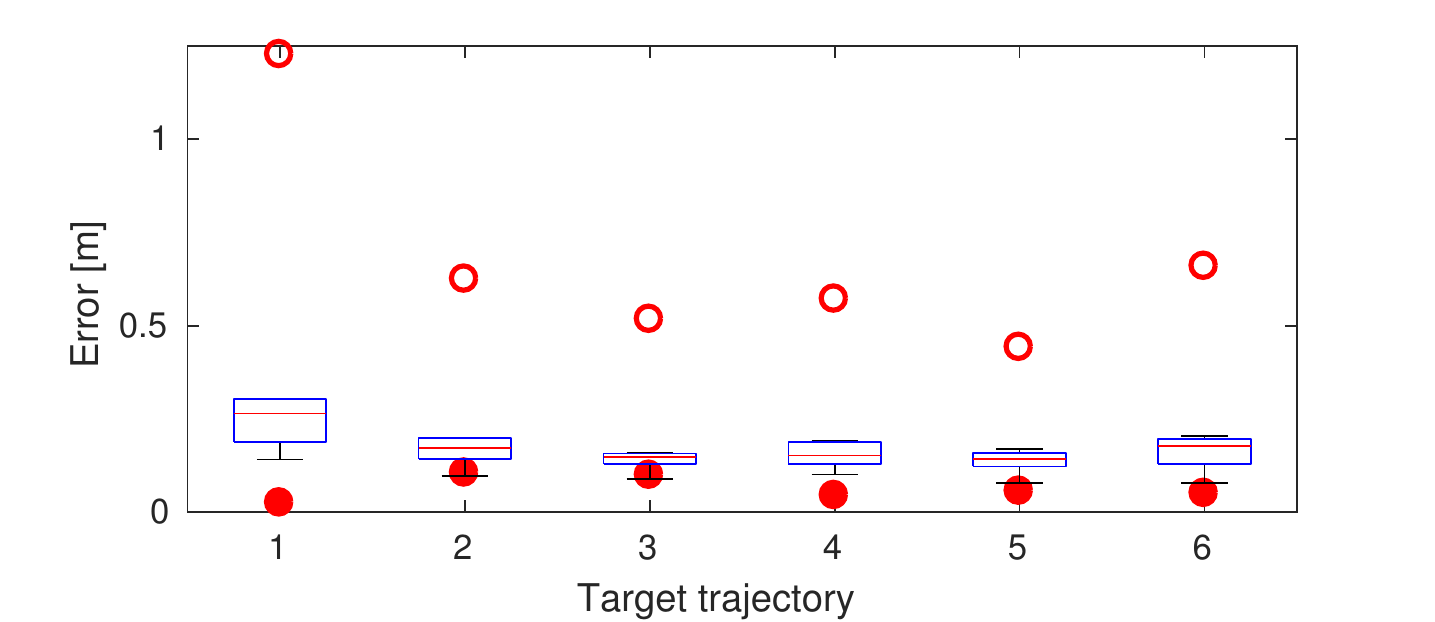}
   \vspace{-9pt}
   \caption{\rev{Tracking errors of six learning processes without transfer information on the Bebop~2 at iteration 1 (open circles) and iteration 10 (filled circles). Boxplots summarize the six one-to-all experiments of the multi-robot, multi-task framework, which significantly decreases the error in the first iteration after transfer. On each box, the red mark indicates the median, while the bottom and top edges indicate the $25^{th}$ and $75^{th}$ percentiles, respectively. The whiskers represent the most extreme data points. }
   }
   \label{fig:all_errors1}
   \vspace{-18pt}
\end{figure}

\subsection{\rev{Multi-robot, multi-task transfer: repeatability}}
\rev{We also test the repeatability of the above results. We choose a single pair of source and target trajectories (in Fig.~\ref{fig:desiredtrajectories3D}, dark blue and orange trajectory, respectively). We repeat ten times a 10-iteration learning process on the Bebop~2 with underlying $\lone$ adaptive controller and ILC when 
\begin{enumerate*}[font=\itshape]
 \item no transfer is used, and 
 \item transfer
\end{enumerate*}
information from the source trajectory learned on the AR.Drone~2.0 is used to initialize the learning process. Fig.~\ref{fig:repeaterror} shows the mean error over the ten repetitions for each of the 10 iterations. The proposed framework significantly decreases the tracking error in the first iteration. The standard deviation is very low, with a mean of 0.0013 m and 0.0016 m for the no transfer and transfer cases, respectively. This is achieved as the $\lone$ adaptive controller forces the Bebop~2 to behave like the reference model despite the presence of disturbances.}

\subsection{\rev{Multi-robot, multi-task transfer: different reference models for the $\lone$ adaptive controllers}}
\rev{For completeness, we include experimental results for Remark~\ref{rem:diffrefsys}, when the source and target systems have $\lone$ adaptive controllers with different reference models. In particular, we modified the reference model of the Bebop 2, and used a mapping between the reference models in addition to the proposed multi-robot, multi-task transfer framework. The tracking error of a 10-iteration ILC process using transfer information is shown in magenta in Fig.~\ref{fig:repeaterror}. The proposed framework reduces the tracking error in the first iteration after transfer by 74.86\%, even when the $\lone$ adaptive controller of the target system has a different reference model.}

\begin{figure}[t]
   \centering
   \includegraphics[width=0.4\textwidth]{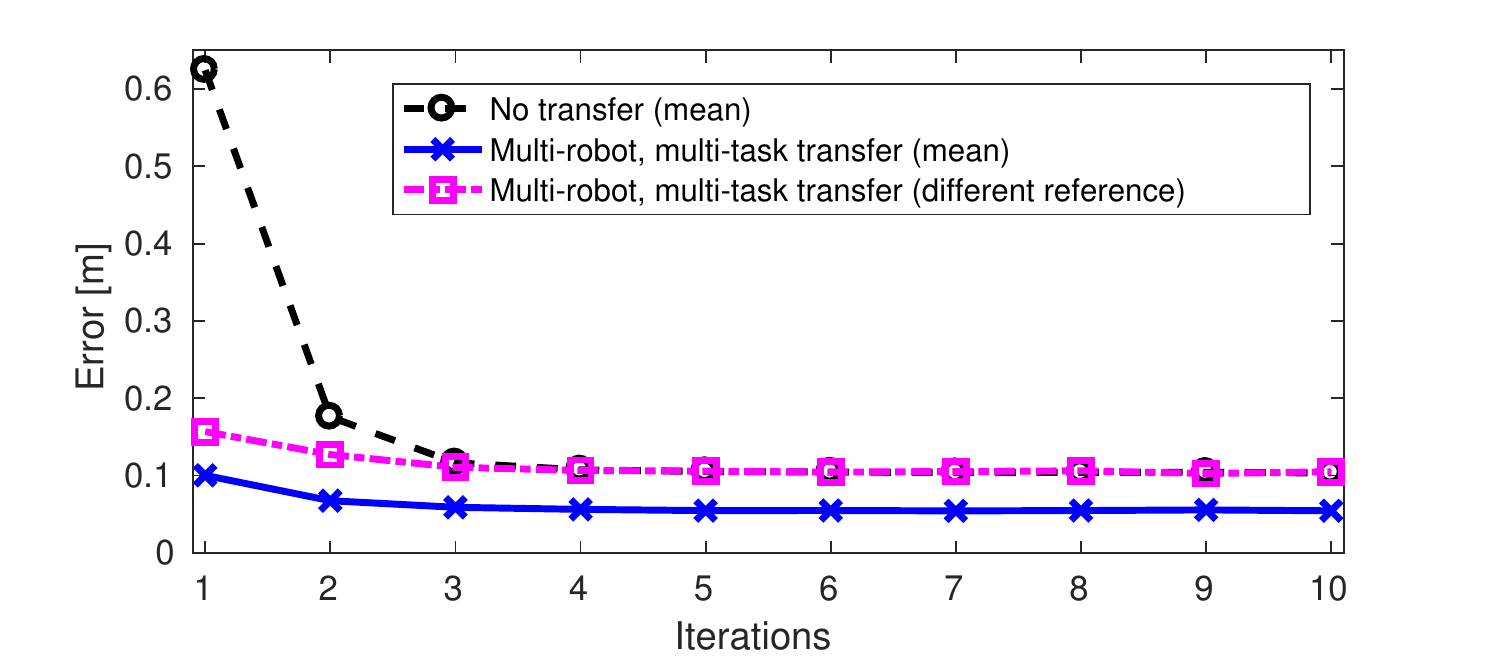}
    \vspace{-9pt}
   \caption{\rev{For a single source and target trajectory pair, the average position errors over ten 10-iteration ILC processes when the Bebop 2 is tracking the target trajectory are shown  
   \emph{(i)} in blue, when the proposed multi-robot (AR.Drone~2.0 to Bebop~2), multi-task (source to target trajectory) framework is used, and
   \emph{(ii)} in black, when no transfer information is used. The errors after the proposed multi-robot, multi-task transfer when the source and target systems have $\lone$ adaptive controllers with different reference models are shown in magenta. 
   }}
   \label{fig:repeaterror}
   \vspace{-14pt}
\end{figure}

\section{CONCLUSIONS}
\label{sec:conclusions}
In this paper, we introduced a multi-robot, multi-task transfer learning framework for MIMO systems. We focused on the trajectory tracking problem. The multi-robot transfer learning framework is based on a combined $\lone$ adaptive controller and ILC. The $\lone$ adaptive controller forces systems to behave like a specified linear reference model, allowing learned tasks to be directly transferred to other systems. The multi-task transfer learning framework uses control theory results to build a time- and state-invariant map from the desired trajectory to the input that accurately tracks this trajectory. This map can be used to generate inputs for new desired trajectories. \rev{Experimental results on two different quadrotors and six different trajectories show that the proposed framework reduces the first-iteration tracking error by 74\% on average, when information from tracking a different, single trajectory on a different quadrotor is utilized.}



%
%
%

\bibliographystyle{IEEEtran}
\bibliography{IEEEabrv,rootarxiv}

\end{document}